%% file: Dropout_GRPO_arxiv.tex
\documentclass{article}

\usepackage{PRIMEarxiv}

\usepackage[utf8]{inputenc}
\usepackage[T1]{fontenc}
\usepackage[numbers,compress]{natbib}
\usepackage{hyperref}
\usepackage{url}
\usepackage{booktabs}
\usepackage{amsmath,amssymb,amsthm}
\usepackage{amsfonts}
\usepackage{mathtools}
\usepackage{bm}
\usepackage{nicefrac}
\usepackage{microtype}
\usepackage{xcolor}
\usepackage{enumitem}
\usepackage{graphicx}
\usepackage{algorithm}
\usepackage{algpseudocode}

\newtheorem{theorem}{Theorem}
\newtheorem{lemma}[theorem]{Lemma}
\newtheorem{proposition}[theorem]{Proposition}
\newtheorem{corollary}[theorem]{Corollary}
\theoremstyle{definition}

\newtheorem{assumption}[theorem]{Assumption}
\theoremstyle{remark}
\newtheorem{remark}[theorem]{Remark}

\newcommand{\R}{\mathbb{R}}
\newcommand{\E}{\mathbb{E}}
\newcommand{\Var}{\mathrm{Var}}

\newcommand{\KL}{\mathrm{KL}}

\newcommand{\D}{\mathcal{D}}
\newcommand{\Loss}{\mathcal{L}}
\newcommand{\eps}{\varepsilon}
\newcommand{\old}{\mathrm{old}}

\title{Dropout-GRPO: Variational Stochasticity for Continuous Latent Reasoning}

\author{
  Wooil Jung \\
  Department of Electrical and Computer Engineering\\
  University of California, San Diego\\
  9500 Gilman Dr., La Jolla, CA 92093 \\
  \texttt{w3jung@ucsd.edu} \\
}

\begin{document}
\maketitle

\begin{abstract}
\input{text/0_abstract}

\end{abstract}


\section{Introduction}
\label{sec:intro}
\input{text/1_introduction}

\section{Related Work}
\label{sec:related}
\input{text/2_related_works}

\section{Theoretical Framework}
\label{sec:theory}
\input{text/3_theoretical_framework}

\section{Methodology: Dropout-GRPO}
\label{sec:method}
\input{text/4_methodology}

\section{Experiments}
\label{sec:setup}
\input{text/5_experiments}

\section{Discussion and Limitations}
\label{sec:disc}
\input{text/6_limitations}

\section{Conclusion}
\label{sec:concl}
\input{text/7_conclusion}

\bibliographystyle{plainnat}
\bibliography{references}

\appendix

\section{Extended Discussion and Future Work}
\label{app:extended}
\input{text/A_extended}

\section{Proofs}
\label{app:proofs}
\input{text/B_proofs}

\end{document}

%% file: text/0_abstract.tex

Group Relative Policy Optimization (GRPO) relies on the diversity of $K$ rollouts  within each group; otherwise, the group-mean advantage $A^{(k)} = r^{(k)} - \mu_r$  collapses to zero. This presents a structural challenge for latent-reasoning  models like \textsc{Coconut}, which feed continuous hidden states recurrently in  place of discrete chain-of-thought tokens. Because the latent phase is inherently  deterministic given the parameters and prompt, multiple rollouts produce identical  trajectories, stalling GRPO's progress. Consequently, applying group-relative  reinforcement learning to continuous latent reasoning has proven difficult.

To address this, we propose sourcing the necessary stochasticity through  structured dropout. By applying a single Bernoulli mask held constant across  
all latent recurrence steps for a given rollout, we generate essential trajectory  variance. This shared mask effectively treats each rollout as a posterior sample  
from a variational distribution over parameters, allowing GRPO to optimize the  expected reward of a Bayesian model-average policy. We provide both theoretical  
justification for this method---including unbiasedness, variance reduction, and  the well-definedness of the latent gradient---and empirical validation. On \textsc{GSM8K}, 
dropout-GRPO improves a \textsc{Coconut} baseline from $27.29\%$ to $29.01\%$ pass@1, demonstrating the viability of GRPO learning for latent-reasoning models. Our work positions this as a practical, 
theoretically grounded approach for post-training latent-reasoning LLMs.

%% file: text/1_introduction.tex

The emergence of Chain-of-Thought (CoT) reasoning\cite{Wei2022ChainOfThought, Zhou2022LeastToMost, Feng2023TheoreticalCoT, Kojima2022ZeroShotReasoners} has significantly advanced the reasoning capabilities of Large Language Models (LLMs). However, traditional CoT methods rely on natural language to articulate intermediate reasoning steps. This approach allocates equivalent compute budgets to both the generation of linguistically fluent text and the execution of core logical operations. Furthermore, decoding continuous latent states into discrete language tokens inevitably incurs an information bottleneck, losing internal state information. To address these limitations, continuous latent reasoning models~\cite{hao2024coconut, Zhangetal2025} have been proposed. By reasoning directly within the model's hidden representation, these approaches compress lengthy reasoning chains into efficient latent steps while preserving rich distributional information, thereby unlocking greater reasoning potential.

Despite these advantages, a central challenge in latent reasoning is transitioning the model's semantic knowledge from a discrete language space to a continuous latent space. Various strategies have been explored to bridge this gap, including curriculum learning, knowledge distillation\cite{Deng2023ImplicitCoT}, continuous or compressed thought representations~\cite{Cheng2024CompressedCoT, Shen2025CODI}, recurrent or dynamically allocated latent computation~\cite{Saunshi2025LoopedTransformers, Chen2025InnerThinking}, and latent-space optimization~\cite{Du2025LatentThinking}. Furthermore, reasoning in the hidden state enables step-level exploration capabilities, effectively encoding multiple potential trajectories simultaneously within a single latent representation.\cite{Zhuetal2025, Gozeten2025ParallelExploration} A prominent example of this paradigm is \textsc{Coconut}~\cite{hao2024coconut}, which leverages curriculum learning to internalize reasoning steps.\cite{Deng2024InternalizeCoT} While \textsc{Coconut} demonstrates impressive logical reasoning, empirical observations suggest its mathematical capabilities can regress during curriculum training, highlighting a critical need for post-training reinforcement learning (RL) to refine its performance.

Reinforcement learning from verifier rewards has established itself as the dominant post-training recipe for reasoning language models on mathematics and code benchmarks~\cite{shao2024deepseekmath,yu2024dapo}. Among these methods, Group Relative Policy Optimization (GRPO)~\cite{shao2024deepseekmath} has emerged as a particularly lightweight and effective choice. It eliminates the learned value network of standard PPO~\cite{schulman2017ppo} by estimating the per-prompt advantage from a small group of rollouts using a simple group-mean baseline.

However, applying GRPO to latent reasoning models presents a fundamental structural challenge. Token-level decoders naturally derive the necessary rollout diversity from stochastic sampling over the next-token distribution. In contrast, the continuous latent recurrence in models like \textsc{Coconut} is inherently deterministic given the prompt and parameters. Consequently, multiple rollouts of the same prompt yield identical latent trajectories, causing the group-mean advantage to collapse to zero and stalling the optimization process entirely.

\paragraph{Our contribution.}
To bridge this gap, we propose sourcing the missing stochasticity through structured \emph{dropout}. In our approach, each rollout is executed with a single Bernoulli mask drawn afresh and held constant across all latent recurrence steps. This mask is recorded and \emph{replayed} during the policy update, ensuring the gradient is computed against the exact trajectory that was evaluated. Following the variational interpretation of~\citet{gal2016dropout}, this shared mask treats each rollout as a single posterior sample from a structured distribution over parameters. Under this construction, GRPO successfully optimizes the expected reward of a Bayesian model-average policy. While modern LLM training has largely abandoned dropout as a regularizer, we observe that its ability to produce bit-identical-replayable structured weight perturbations makes it a natural mechanism for the exogenous-noise role required by our construction. We summarize our contributions as follows:

\begin{enumerate}[topsep=2pt,itemsep=2pt,leftmargin=*]
\item \textbf{A new paradigm for RL on latent-reasoning models.}  We identify the structural reason GRPO fails on \textsc{Coconut}-style policies---deterministic latent recurrence---and propose shared-mask variational dropout with mask replay as a minimal, theoretically clean fix. This is the first method, to our knowledge, that successfully applies group-relative reinforcement learning to a continuous latent-reasoning model that uses hidden states directly in \textsc{Coconut} style.
\item \textbf{Theoretical analysis.}  We prove that the dropout-GRPO surrogate gradient with the mean-only advantage $A^{(k)} = r^{(k)} - \mu_r$ equals $(1 - 1/K)\,\nabla_\theta J(\theta)$ at $\theta = \theta_\old$, where $J$ is the expected reward of the mask-marginalized policy. We further show that mask replay provides a common-random-numbers variance reduction over a fresh-rollout baseline, and that the latent Jacobian $\partial h_T / \partial \theta$ is well defined at every depth $T$ under standard Transformer regularity.
\item \textbf{Empirical validation on GSM8K.}  Dropout-GRPO improves a \textsc{Coconut} baseline from $27.29\%$ to $29.01\%$ pass@1, demonstrating that the proposed construction yields a viable learning signal in the regime where deterministic-rollout GRPO fails entirely.
\item \textbf{Implementation refinements and reference code.}  We document four implementation refinements (Huberized $k_3$ KL, learning-rate-coupled $\beta$ annealing, sequence-mean loss aggregation, and a group-level accuracy filter with DDP-coordinated skip steps) that are adopted to stabilize training, and release a reference implementation that realizes bit-identical mask replay across the entire latent recurrence.
\end{enumerate}

%% file: text/2_related_works.tex
\paragraph{Explicit reasoning and Chain-of-Thought.}
Chain-of-Thought (CoT) prompting~\citep{Wei2022ChainOfThought, Kojima2022ZeroShotReasoners} substantially improves LLM performance on symbolic, logical, and mathematical reasoning by articulating intermediate steps in natural language. The autoregressive generation of these steps incurs severe latency and memory-bandwidth overhead, motivating techniques such as CoT distillation~\citep{Hsieh2023DistillingStepByStep} and prompt compression~\citep{Jiang2023LLMLingua}. These approaches still operate within the discrete-token bottleneck, motivating paradigms that bypass language generation during reasoning entirely.

\paragraph{Latent-space reasoning.}
\citet{hao2024coconut} introduced \textsc{Coconut}, providing a pioneering large-scale demonstration that continuous hidden states could replace discrete chain-of-thought tokens for multi-step reasoning. Subsequent work has explored continuous CoT as a superposition of reasoning traces or parallel search frontiers~\cite{Deng2025VocabularySuperposition, Zhuetal2025, Gozeten2025ParallelExploration}. Other approaches compress explicit CoT into continuous representations via self-distillation~\cite{Shen2025CODI} or optimize hybrid latent reasoning policies with RL~\cite{Yue2025HybridHRPO}. Most relevant to our setting is \textsc{LEPO}\cite{Zhou2026LEPO}, which sources stochasticity through a Gumbel-softmax relaxation for latent reasoning and applies policy optimization directly to the resulting latent trajectories. Our approach differs in two key respects: the stochasticity is parameter-side (a structured weight perturbation) rather than embedding-side, and the resulting marginal policy admits a Bayesian model-average interpretation that LEPO's discrete-relaxation construction does not.

\paragraph{Reinforcement learning for language models.}
Policy-gradient training of language models begins with REINFORCE-style estimators~\citep{williams1992reinforce} and was scaled to modern LLMs by Proximal Policy Optimization (PPO)~\citep{schulman2017ppo} with a learned value baseline.  GRPO \citep{shao2024deepseekmath} dispenses with the value model in favor of a group-mean baseline computed across $K$ rollouts of the same prompt, and has become the workhorse for math reasoning post-training. Recent refinements include \textsc{DAPO}~\citep{yu2024dapo}, which introduces ``dynamic sampling'' to filter zero-variance groups, and ``Dr.\ GRPO''~\citep{liu2024drgrpo}, which observes that the standard within-group standard-deviation normalization $1/(\sigma_r + \delta)$ biases the policy gradient and recommends its removal. Our group-level accuracy filter generalizes the DAPO diagnostic into an active filter, and we adopt the Dr.\ GRPO advantage form $A^{(k)} = r^{(k)} - \mu_r$ throughout.

\paragraph{Variational dropout and Bayesian deep learning.}
\citet{gal2016dropout} establish that dropout can be interpreted as variational inference over a structured posterior on weights. We build directly on this interpretation: by enforcing a shared mask across all $T$ latent recurrence steps within one rollout, each rollout becomes a single posterior sample, and the marginal policy is a Bayesian model average. The stochastic-computation-graph framework of~\citet{schulman2015scg} supplies the formal score-function gradient machinery for objectives with stochastic nodes. Common-random-numbers variance reduction~\citep{glasserman1992crn} provides the basis for our mask-replay argument: replaying the rollout-time mask at update time couples the estimator across nearby parameter values, reducing variance relative to fresh resampling.

%% file: text/3_theoretical_framework.tex
\label{sec:prelim}

We first fix notation for the Coconut latent recurrence, GRPO, and the shared-mask dropout construction (Section~\ref{sec:theory-setup}); we then establish well-definedness of the latent gradient (Section~\ref{sec:theory-welldef}), unbiasedness of the surrogate gradient with the mean-only advantage
(Section~\ref{sec:theory-unbiased}), and a common-random-numbers variance reduction from mask replay
(Section~\ref{sec:theory-crn}).  We commit to a Bayesian model-average reading throughout.

\subsection{Setup}
\label{sec:theory-setup}

\paragraph{Coconut latent recurrence.}
Let $f_\theta : \R^d \to \R^d$ denote one Transformer block of a latent-reasoning language model with parameters $\theta \in \R^D$. Given a prompt $x$, \textsc{Coconut}\cite{hao2024coconut} produces an initial hidden state $h_0 = \mathrm{embed}_\theta(x)$, unrolls $T$ latent recurrence steps $h_t = f_\theta(h_{t-1})$, then decodes answer tokens $y_\ell \sim \pi_\theta(\cdot \mid x, h_T, y_{<\ell})$.  A
verifier supplies $r(y, x) \in [0,1]$ on answer exact-match. Conditional on $\theta$ and $x$ the latent phase is deterministic: $K$ independent rollouts produce $K$ identical $h_T$ and (under low-temperature decoding) $K$ identical answers.

\paragraph{GRPO.}
For a group of $K$ rollouts under $\theta_\old$ with rewards $r^{(k)} = r(y^{(k)}, x)$, the per-rollout advantage is $A^{(k)} = r^{(k)} - \mu_r$ with $\mu_r = \tfrac{1}{K}\sum_k r^{(k)}$; following~\citet{liu2024drgrpo} we omit the $1/(\sigma_r + \delta)$ normalization for reasons given in Remark~\ref{rem:sigma}.  The clipped per-token surrogate \citep{shao2024deepseekmath} is
\begin{equation}
\Loss^{\mathrm{GRPO}}(\theta)
   \;=\;
\E_{x \sim \D}
   \frac{1}{K}\sum_{k=1}^K
   \frac{1}{|y^{(k)}|}\sum_\ell
      \min\!\Big(\rho_{\theta,\ell}^{(k)} A^{(k)},\,
                  \mathrm{clip}(\rho_{\theta,\ell}^{(k)}, 1\pm\eps)\, A^{(k)}\Big)
   - \beta\,\KL\!\big(\pi_\theta\,\big\|\,\pi_{\mathrm{ref}}\big),
\label{eq:theory-grpo-loss}
\end{equation}
with per-token ratio
$\rho_{\theta,\ell}^{(k)} = \pi_\theta(y_\ell^{(k)}\!\mid\!\cdot)/\pi_{\theta_\old}(y_\ell^{(k)}\!\mid\!\cdot)$
and KL anchor $\pi_{\mathrm{ref}}$.

\paragraph{Shared-mask dropout and the augmented policy.}
Let $H$ index the dropout-affected hidden units.  Drawing $\xi \sim p(\xi)$ from a fixed mask distribution that does not depend on $\theta$, dropout induces a structured perturbation $\tilde\theta(\xi)$ of the network's effective weights with $\E_\xi[\tilde\theta(\xi)] = \theta$.  We adapt the variational dropout interpretation of~\citet{gal2016dropout} by fixing one mask per rollout, reused across all $T$ latent steps to produce the \textit{dropout-perturbed recurrence}:
\begin{equation}
h_t \;=\; f_{\tilde\theta(\xi)}(h_{t-1}),
\qquad t = 1, \dots, T,
\label{eq:theory-perturbed}
\end{equation}
so each rollout corresponds to a single posterior sample $\tilde\theta(\xi) \sim q_\theta(\theta')$ from a structured variational distribution. The induced policy treats $\xi$ as part of the action, so the induced distribution over $(\xi, y)$ given $x$ is
\begin{equation}
\Pi_\theta(\xi, y \mid x) \;=\; p(\xi)\,\pi_\theta(y \mid x, \xi),
\qquad
\pi_\theta(y \mid x, \xi)
   \;=\; \prod_\ell \pi_\theta\!\big(y_\ell \,\big|\, x, h_T(\theta, \xi, x), y_{<\ell}\big),
\label{eq:theory-augmented-policy}
\end{equation}
with marginal policy $\bar\pi_\theta(y \mid x) := \E_{\xi}[\pi_\theta(y \mid x, \xi)] = \E_{\theta' \sim q_\theta}[\pi_{\theta'}(y \mid x)]$, the Bayesian model-average policy. We optimize $J(\theta) := \E_{x \sim \D}\, \E_{(\xi, y) \sim \Pi_\theta}[r(y, x)]$, which by construction equals the expected reward of $\bar\pi_\theta$.

\paragraph{Mask replay.}
At update time we regenerate $\xi^{(k)}$ from a stored RNG seed and recompute $\log \pi_\theta(y^{(k)} \mid x, \xi^{(k)})$ under the \emph{same} mask used at rollout time.  This makes $h_T$ at update time match its rollout-time value bit-identically when $\theta = \theta_\old$, so that $\rho_{\theta_\old, \ell}^{(k)} \equiv 1$ exactly.

\subsection{Well-definedness of the latent gradient}
\label{sec:theory-welldef}

\begin{assumption}\label{ass:smooth}
For every fixed $\xi$ and almost every $\theta \in \R^D$, the map $\theta \mapsto f_{\tilde\theta(\xi)}(h)$ is differentiable in $\theta$ for almost every $h$, and the answer-token log-likelihood is differentiable in both $\theta$ and $h_T$.
\end{assumption}

The assumption holds for standard Transformer blocks: the dropout rescaling $\theta \mapsto \tilde\theta(\xi)$ is linear in $\theta$ and the remaining nonlinearities (LayerNorm, GELU, softmax) are almost everywhere differentiable.

\begin{lemma}[Well-defined latent gradient]
\label{lem:induction}
Under Assumption~\ref{ass:smooth}, for every $T \ge 0$ and every $\xi$, the Jacobian $\partial h_T / \partial \theta$ exists almost everywhere and is given by the recursion
\begin{equation}
\frac{\partial h_t}{\partial \theta}
   \;=\;
\frac{\partial f_{\tilde\theta}}{\partial \tilde\theta}\bigg|_{h_{t-1}}
   \cdot \frac{\partial \tilde\theta}{\partial \theta}
   \;+\;
\frac{\partial f_{\tilde\theta}}{\partial h}\bigg|_{h_{t-1}}
   \cdot \frac{\partial h_{t-1}}{\partial \theta},
\qquad
\frac{\partial h_0}{\partial \theta}
   \;=\;
\frac{\partial \mathrm{embed}_\theta(x)}{\partial \theta},
\label{eq:theory-jacobian}
\end{equation}
so $\nabla_\theta \log \pi_\theta(y \mid x, \xi)$ is well defined almost everywhere.
\end{lemma}

\noindent Proof in Appendix~\ref{app:proofs}. The recursion~\eqref{eq:theory-jacobian} is what \texttt{loss.backward()} computes when the dropout mask is held constant via RNG-state restoration: PyTorch's \texttt{torch.nn.functional.dropout} treats the mask as a constant during backprop, so $\partial \tilde\theta / \partial \theta$ is the structured Bernoulli-rescaled identity exactly.

\subsection{Unbiasedness of the surrogate gradient}
\label{sec:theory-unbiased}

\begin{theorem}[Unbiased policy-surrogate gradient]
\label{thm:unbiased}
Treat the $K$ rollouts as i.i.d.\ draws from $\Pi_{\theta_\old}(\cdot \mid x)$ and define the policy-surrogate
component $\Loss^{\mathrm{surr}} := \Loss^{\mathrm{GRPO}} + \beta \KL$. Under $A^{(k)} = r^{(k)} - \mu_r$ and mask replay,
\begin{equation}
\nabla_\theta \Loss^{\mathrm{surr}}(\theta)\Big|_{\theta = \theta_\old}
   \;=\;
\Big(1 - \tfrac{1}{K}\Big)\,\nabla_\theta J(\theta)\Big|_{\theta_\old}.
\label{eq:theory-unbiased}
\end{equation}
\end{theorem}

\begin{proof}
See Appendix~\ref{app:proofs}.
\end{proof}

\begin{corollary}[Optimization on the marginal policy]
\label{cor:marginal}
The estimator targets $\nabla J(\theta)$, the expected reward of the Bayesian model-average policy $\bar\pi_\theta$ \eqref{eq:theory-augmented-policy}.  The shared-mask construction turns dropout-GRPO into ensemble-aware policy optimization: each gradient step is a Monte Carlo average of per-ensemble-member policy gradients, with the group mean serving as a control variate.
\end{corollary}

\begin{remark}[Why $\sigma_r$ is dropped]
\label{rem:sigma}
The original GRPO advantage $(r^{(k)} - \mu_r)/(\sigma_r + \delta)$ couples $s^{(k)}$ to every $r^{(j)}$ via a nonlinear function of all rewards; the factorization $\E[r^{(j)} s^{(k)}] = \E[r^{(j)}]\E[s^{(k)}]$ used above breaks and the estimator is biased. \citet{liu2024drgrpo} report that this bias overweights low-variance prompts and underweights high-variance ones; we drop $\sigma_r$ throughout.
\end{remark}

\subsection{Variance reduction via mask replay}
\label{sec:theory-crn}

Two unbiased estimators of $\nabla J(\theta_\old + \Delta)$ for small $\Delta$ are available: the \emph{mask-replay} (CRN) estimator $\hat g_{\mathrm{CRN}}(\theta;\xi,y) := A(\xi,y)\,\nabla_\theta \log \pi_\theta(y \mid x, \xi)$, which reuses the rollout-time $(\xi, y)$, and the \emph{fresh-rollout} estimator $\hat g_{\mathrm{fresh}}(\theta;\xi',y')$ obtained by drawing $(\xi', y') \sim \Pi_{\theta_\old}(\cdot \mid x)$ at update time and recomputing the advantage from a fresh group.

\begin{proposition}[CRN coupling]
\label{prop:crn}
Both estimators are unbiased for $\nabla J(\theta)$ at $\theta = \theta_\old$ (modulo the $(1 - 1/K)$ scale of Theorem~\ref{thm:unbiased}). Suppose $\theta \mapsto A(\xi, y)\,\nabla \log \pi_\theta(y \mid x, \xi)$ is $L$-Lipschitz $\Pi_{\theta_\old}$-a.s.  Then in the trust-region regime $\|\Delta\| \to 0$ the two estimators have matching leading-order variance, while the fresh-rollout estimator carries an additional across-call resampling-variance term that the mask-replay estimator eliminates by construction.
\end{proposition}

\begin{proof}[Proof sketch]
See Appendix~\ref{app:proofs}.
\end{proof}

In the practical PPO setting where $\|\theta - \theta_\old\|$ is small under the clip and $E$ mini-epochs of updates are taken per rollout batch, the coupling is the dominant variance reduction.

%% file: text/4_methodology.tex
This section describes the dropout-GRPO algorithm
(Section~\ref{sec:method-algo}) and four implementation refinements
adopted to stabilize training in practice
(Sections~\ref{sec:method-huber}--\ref{sec:method-filter}).

\subsection{Algorithm}
\label{sec:method-algo}

The procedure is summarized in Algorithm~\ref{alg:dropout-grpo}.  For
each prompt $x$, we draw $K$ independent dropout masks
$\xi^{(1)}, \dots, \xi^{(K)}$, each held constant across all $T$ latent
recurrence steps within its rollout, and store the underlying RNG
seed.  We compute the rollout reward $r^{(k)}$, advantage
$A^{(k)} = r^{(k)} - \mu_r$, and at update time regenerate
$\xi^{(k)}$ from the stored seed (via PyTorch's CUDA + CPU RNG-state
restoration), so that $\log \pi_\theta(y^{(k)} \mid x, \xi^{(k)})$ at
update time matches its rollout-time value bit-identically when
$\theta = \theta_\old$. We use sequence-mean for consistency with per-rollout normalization.

\begin{algorithm}[h]
\caption{Dropout-GRPO for latent-reasoning LLMs}
\label{alg:dropout-grpo}
\begin{algorithmic}[1]
\Require $\pi_{\theta_\old}$, $\D$, $K$, $p$, $T$, $\eps$, $\beta_0$, $\eta_t$, $[\alpha_{\min}, \alpha_{\max}]$, $\delta$.
\For{$t = 1, \dots, T_{\max}$}
  \State Sample mini-batch $x \sim \D$. For $k = 1, \dots, K$: store RNG seed for $\xi^{(k)}$, unroll \eqref{eq:theory-perturbed} restoring the chain seed at every latent step (one mask per rollout), decode $y^{(k)}$, record $\log \pi_{\theta_\old}(y^{(k)} \mid x, \xi^{(k)})$, compute $r^{(k)}$.
  \State Compute $\mu_r$. If $\mu_r \notin [\alpha_{\min}, \alpha_{\max}]$, drop group (DDP-coordinated, \S\ref{sec:method-filter}); else set $A^{(k)} = r^{(k)} - \mu_r$.
  \State For each kept $k$: regenerate $\xi^{(k)}$ via RNG restore, recompute $\log \pi_\theta(y^{(k)} \mid x, \xi^{(k)})$ and $\rho_\ell^{(k)}$.
  \State Take a gradient step on
         $\Loss(\theta) = -\frac{1}{K_{\mathrm{kept}}}\sum_{k} \frac{1}{|y^{(k)}|}\sum_\ell \min(\rho_\ell^{(k)} A^{(k)}, \mathrm{clip}(\rho_\ell^{(k)}, 1\pm\eps) A^{(k)}) + \beta(t)\frac{1}{K_{\mathrm{kept}}}\sum_k \frac{1}{|y^{(k)}|}\sum_\ell \tilde k_3(r_\ell^{\mathrm{ref}})$,
         with $\beta(t) = \beta_0\,\eta_t/\eta_{\max}$ and $\tilde k_3$ from~\eqref{eq:method-huber-kl}; set $\theta_\old \gets \theta$.
\EndFor
\end{algorithmic}
\end{algorithm}

\begin{figure}[t]
\centering
\includegraphics[width=\linewidth]{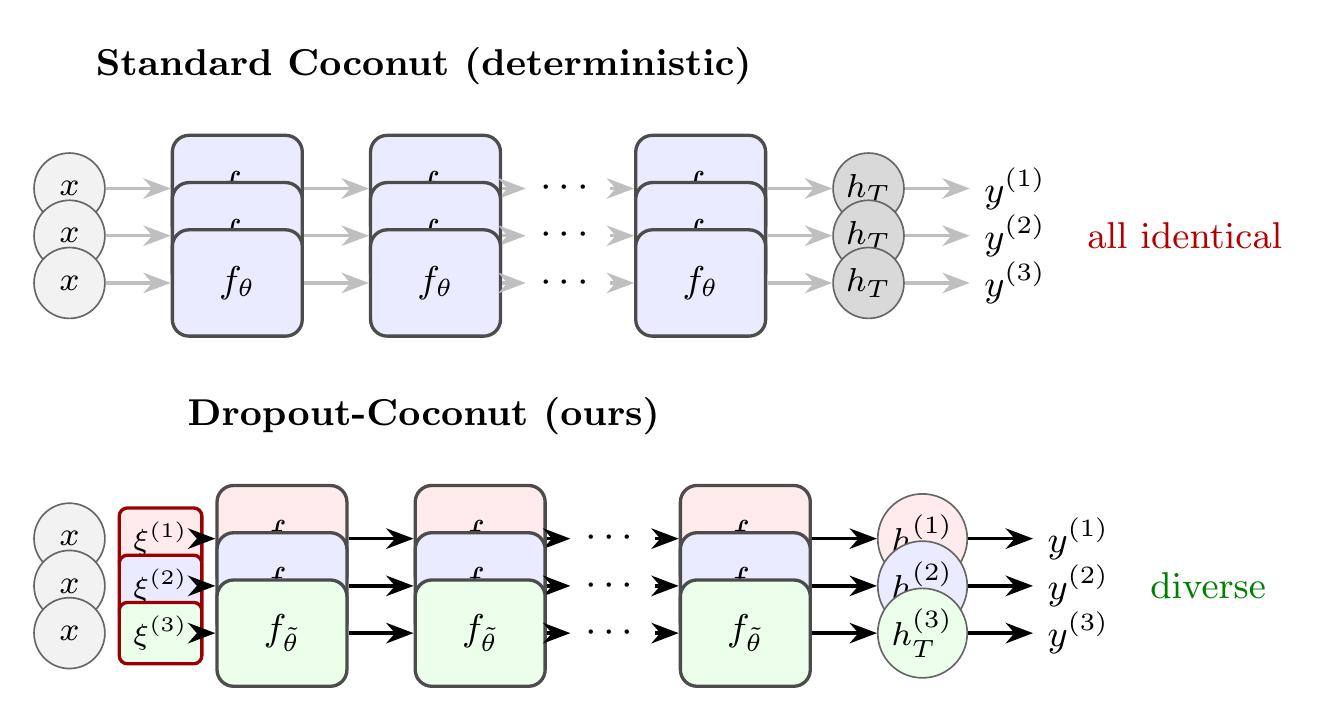}
\caption{Dropout-GRPO pipeline. A single prompt $x$ is broadcast to $K$ parallel rollouts; each rollout draws an independent mask $\xi^{(k)}$ that is held constant across all $T$ latent steps (denoted by $f^{\times T}$). The resulting answers $y^{(k)}$ receive verifier rewards $r^{(k)}$, the group mean $\mu_r$ is computed, and each rollout's advantage calculated by $A^{k} = r^{(k)} - \mu_r$. Mask replay at update time ensures the policy gradient is computed against the trajectory that was graded.}
\label{fig:method-overview}
\end{figure}

\subsection{Huberized $k_3$ KL estimator}
\label{sec:method-huber}

The $k_3$ KL estimator~\citep{schulman2020kl} is
$k_3(r) = \exp(r) - r - 1$ with
$r := \log \pi_\theta - \log \pi_{\mathrm{ref}}$.  A common
implementation choice is to clamp $r$ to $[-\delta, \delta]$ before
evaluating $k_3$ to prevent overflow.  This has the unintended effect
of zeroing the gradient outside $[-\delta, \delta]$ (the autograd
derivative of \texttt{clamp} is zero beyond its bounds), so any token
whose log-ratio drifts past $\pm\delta$ escapes KL regularization
entirely.  Empirically, this is a dominant pre-collapse signature: a
small tail of escaped tokens drifts further with no restoring force,
eventually dragging the whole policy with it.

We replace the clamp with a Huber-style extension: $k_3$ inside,
linear extrapolation outside, value- and slope-matched at $\pm\delta$:
\begin{equation}
\tilde k_3(r) \;=\;
\begin{cases}
\exp(r) - r - 1 & |r| \le \delta, \\
(e^\delta - 1)(r - \delta) + (e^\delta - \delta - 1) & r > \delta, \\
(e^{-\delta} - 1)(r + \delta) + (e^{-\delta} + \delta - 1) & r < -\delta.
\end{cases}
\label{eq:method-huber-kl}
\end{equation}
By construction $\tilde k_3 \in C^1(\R)$, $\tilde k_3 \ge 0$, and the
gradient $\tilde k_3'(r)$ is bounded by $|e^\delta - 1|$ in the right
tail and $|e^{-\delta} - 1| < 1$ in the left tail but \emph{nonzero}
everywhere outside $r = 0$.  Drifted tokens therefore continue to feel
a restoring pull, while the maximum per-token contribution stays
bounded.  We use $\delta = 5$ throughout, giving a maximum slope of
$e^5 - 1 \approx 147$.

\subsection{Trust-region annealing of $\beta$}
\label{sec:method-anneal}

Under a cosine-annealed learning-rate schedule $\eta_t$, holding the
KL coefficient $\beta$ fixed produces a slowly worsening dynamic: the
per-step policy update shrinks with $\eta_t$, but the per-step KL pull
$\eta_t \cdot \beta \cdot \nabla \KL$ also shrinks linearly with
$\eta_t$, so the relative pull stays constant.  However, the policy
gradient direction integrates informatively across steps while the KL
gradient adds coherently toward $\theta_{\mathrm{ref}}$, eventually
pinning the policy at an equilibrium where the two cancel.  We anneal
$\beta$ proportionally with $\eta_t$,
\begin{equation}
\beta(t) \;=\; \beta_0 \cdot \frac{\eta_t}{\eta_{\max}},
\label{eq:method-anneal}
\end{equation}
so the per-step KL pull scales with $\eta_t^2$.  This shifts the
trust-region equilibrium outward as the lr decays and lets the policy
continue to drift away from $\pi_{\mathrm{ref}}$ on the slow
schedule.

\subsection{Group-level accuracy filter and emergent curriculum}
\label{sec:method-filter}

Following~\citet{yu2024dapo}'s ``dynamic sampling'' diagnostic, we
observe that groups with $\mu_r \in \{0, 1\}$ on binary verifier
rewards have $A^{(k)} \equiv 0$ for every $k$ and contribute no
gradient signal.  We extend this diagnostic into an active filter: at
the end of the rollout phase, drop groups whose empirical accuracy
$\mu_r$ falls outside a configured window
$[\alpha_{\min}, \alpha_{\max}]$.  Default values are
$\alpha_{\min} = 0.20$ and $\alpha_{\max} = 0.95$.  Dropping these
groups before any backward pass concentrates compute on the
marginal band where signal-to-noise is highest.  Because the
band shifts as the policy improves, an emergent online curriculum
appears with no explicit data ordering.

\paragraph{Compatibility with Theorem~\ref{thm:unbiased}.}
The filter is a function of the rollout outcomes $\{r^{(k)}\}$, which
depend on $\theta$, and is therefore not exogenous.  The unbiasedness
statement of Theorem~\ref{thm:unbiased} no longer applies to
$\nabla J(\theta)$ on the original prompt distribution but does apply
to the filtered objective
$J_{\mathrm{filt}}(\theta) := \E_{x \sim \D}\big[\E_{(\xi, y)}[r] \mid \mu_r(x) \in [\alpha_{\min}, \alpha_{\max}]\big]$.
The implicit reweighting toward the marginal-difficulty band is the
curriculum effect, by design.

%% file: text/5_experiments.tex
\subsection{Model}
\label{sec:setup-model}

We use \texttt{Qwen2.5-1.5B}~\citep{qwen2024} as the base model for all experiments. We deviate from the GPT-2-based backbone of the original \textsc{Coconut} paper~\citep{hao2024coconut} for two reasons: (i) GPT-2 uses absolute positional embeddings, which combined with left-padding induces a position bias that is particularly harmful in the latent-recurrence regime where prompt length varies across rollouts; (ii) Qwen2.5-1.5B provides an optimal balance by retaining modern architectural choices (RoPE, GQA, SwiGLU) and sufficient representational capacity for policy learning, while remaining tractable on a single GPU for the $K \cdot T$ rollout compute that GRPO requires.

We compare the following models:
\begin{itemize}[topsep=2pt,itemsep=2pt,leftmargin=*]
\item \textbf{Qwen2.5-1.5B base} (zero-shot and 1-shot).
\item \textbf{Qwen2.5-1.5B + Coconut SFT.}  \textsc{Coconut}-style supervised fine-tuning following the curriculum of~\citet{hao2024coconut}, at latent depth $T = 6$. Curriculum training uses $c=3$ latent tokens per reasoning step.
\item \textbf{Coconut + GRPO (deterministic rollouts).}  Standard GRPO without dropout.  Expected to fail by the argument of Section~\ref{sec:intro}; included to demonstrate the structural barrier.
\item \textbf{Coconut + REINFORCE + dropout (EMA baseline).}  Single rollout per prompt ($K = 1$) with a running EMA baseline in place of the group-mean baseline.  Isolates the contribution of group-relative structure from the contribution of latent stochasticity.
\item \textbf{Coconut + dropout-GRPO (ours).}  The full method of Algorithm~\ref{alg:dropout-grpo}.
\end{itemize}

\subsection{Datasets}
\label{sec:setup-data}

The dataset used and the respective number of prompts are detailed in Table~\ref{tab:train_data}. The choice of SFT dataset is consistent with the baseline \textsc{Coconut} model, which provides concise per-step intermediate reasoning suitable for the \textsc{Coconut} curriculum. We deliberately match the SFT and GRPO data distributions in style and difficulty: the \textsc{Coconut} latent representation is sensitive to the prompt distribution, and large distributional shift between SFT and GRPO degrades the SFT-initialized latent representation faster than GRPO can productively refine it.  Concise-step math-reasoning datasets in the \textsc{GSM8K-Aug} mold are scarce, which limits the scope of dataset-side ablations.

\subsection{Training details}
\label{sec:setup-training}

Hyperparameters are detailed in Table~\ref{tab:hparams}.  We train for $3000$ GRPO steps with a batch size of $16$ prompts and $K = 32$ rollouts per prompt, dropout rate $p = 0.1$, peak learning rate $\eta_{\max} = 10^{-5}$ with $200$-step warmup and cosine decay, peak KL coefficient $\beta_0 = 5 \times 10^{-3}$ with proportional annealing per~\eqref{eq:method-anneal}, group filter window $[\alpha_{\min}, \alpha_{\max}] = [0.20, 0.95]$, and Huber threshold $\delta = 5$.

We load and train the model weights in bfloat16 precision and optimizer states in 8-bit, to maximize group size $K$ in a given compute budget. All reported experiments were conducted on a single GPU. The implementation includes the DDP-coordinated skip mechanism described in Section~\ref{sec:method-filter} as a forward-compatible hook for multi-rank training, but multi-GPU operation has not been empirically validated; all reported numbers come from single-GPU runs.

\begin{table}[h]
\centering
\begin{minipage}[t]{0.48\linewidth}
\centering
\caption{Training data sources and volume}
\label{tab:train_data}
\begin{tabular}{lr}
\toprule
Dataset & Usage Vol. \\
\midrule
GSM8K-aug(SFT)\cite{Deng2024InternalizeCoT} & 385620 \\
GSM8K-hard(GRPO)\cite{gao2022pal} & 1319\\
SVAMP(GRPO)\cite{Patel2021AreNLPModels} & 700\\
OpenMathInstruct-2(GRPO)\cite{Toshniwal2024OpenMathInstruct2} &  69981\\
GSM8K(Testing) &  1319\\
MultiArith(Testing)\cite{Roy2015SolvingGeneral} &  600\\
SVAMP(Testing) &  300\\
\bottomrule
\end{tabular}
\end{minipage}
\hfill
\begin{minipage}[t]{0.48\linewidth}
\centering
\caption{Training hyperparameters.}
\label{tab:hparams}
\begin{tabular}{lr}
\toprule
Parameter & Value \\
\midrule
Base model & Qwen Base \\
Latent depth $T$ & $6$ \\
Dropout rate $p$ & $0.1$ \\
Group size $K$ & $32$ \\
Train batch size (prompts) & $16$ \\
Peak learning rate $\eta_{\max}$ & $1 \times 10^{-5}$ \\
Warmup steps & $200$ \\
Total GRPO steps & $3000$ \\
Peak KL coefficient $\beta_0$ & $5 \times 10^{-3}$ \\
PPO clip $\eps$ (lower / upper) & $0.10$ / $0.12$ \\
Filter window $[\alpha_{\min}, \alpha_{\max}]$ & $[0.20, 0.95]$ \\
Huber threshold $\delta$ & $5$ \\
Reference renewal interval & $500$ steps \\
\bottomrule
\end{tabular}
\end{minipage}
\end{table}



\subsection{Results}
We report the best post-warmup (step $\ge 200$) validation pass@1 accuracy for all methods. Table~\ref{tab:main} details these results across multiple test sets, evaluated greedily without dropout. Dropout-GRPO achieves 29.01\% pass@1, outperforming the Coconut-SFT baseline (27.29\%) by 1.72 points. This marks the first successful application of group-relative reinforcement learning to a continuous latent-reasoning model.

Furthermore, Dropout-GRPO reverses the 1.33-point regression on SVAMP caused by Coconut curriculum training. It recovers performance from 41.67\% to 44.00\% $\pm$ 0.58\%, matching the Qwen2.5-1.5B base model within seed-level variance and demonstrating that lost math capabilities are recoverable via RL post-training rather than permanently degraded.

Figure~\ref{fig:eval_graph} illustrates the training trajectories. Standard GRPO with deterministic rollouts yields a flat trajectory at the SFT initialization, confirming the structural barrier noted in Section~\ref{sec:intro}. Without variance ($\sigma_r \equiv 0$), the per-rollout advantage $A^{(k)}$ and policy gradient vanish. Only the weak KL anchor pull remains, resulting in no measurable learning signal for Coconut-style latent reasoning.

Conversely, REINFORCE+dropout fails due to high variance. While latent stochasticity ensures diverse rollouts, replacing the group-relative baseline with a running EMA ($\alpha = 0.95$) degrades the gradient signal-to-noise ratio. At $K=1$, the single-rollout gradient variance exceeds the policy's stability margin, destabilizing the SFT initialization faster than it can learn (visible in the early decline in Figure~\ref{fig:eval_graph}). This highlights that GRPO's group baseline is essential for variance reduction.

Finally, Figure~\ref{fig:std_r_graph} validates the necessity of within-group reward standard deviation ($\sigma_r > 0$) for non-trivial GRPO gradient flow. Without dropout, all $K$ rollouts are bit-identical, collapsing $\sigma_r$ to exactly zero and triggering the structural failure mode our work addresses. The REINFORCE ablation ($K=1$) leaves within-group variance undefined, leading to the aforementioned destabilization. Ultimately, our shared-mask dropout successfully supplies the rollout diversity required by GRPO, whereas removing it entirely eliminates the gradient signal.

\begin{table}[b]
\centering
\caption{Accuracy with pass@1 across model variants.  Boldface indicates the best result. The Coconut baseline uses a single seed due to high computational costs in training; other experiments use 3 seeds.}
\label{tab:main}
\resizebox{\linewidth}{!}{%
\begin{tabular}{lccc}
\toprule
Model & GSM8K & SVAMP & MultiArith\\
\midrule
Qwen2.5-1.5B (zero-shot)                & $62.47\%$             & $43.00\%$             & $82.50\%$ \\
\midrule
Coconut SFT (baseline)                  & $27.29\%$             & $41.67\%$             & $61.67\%$ \\
Coconut + GRPO(deterministic rollout)   & $27.29\%$             & $41.67\%$             & $61.67\%$ \\
Coconut + REINFORCE + dropout (EMA)     & $25.09\% \pm 0.39\%$  & $38.78\% \pm 0.84\%$  & $57.39\% \pm 0.70\%$ \\
\textbf{Coconut + dropout-GRPO (ours)}  & $\mathbf{29.01\% \pm 0.18\%}$ & $\mathbf{44.00\% \pm 0.58\%}$ & $\mathbf{64.67\% \pm 1.01\%}$ \\
\bottomrule
\end{tabular}%
}
\end{table}

\begin{figure}[h]
    \centering
    \begin{minipage}[t]{0.48\linewidth}
        \centering
        \includegraphics[width=\linewidth]{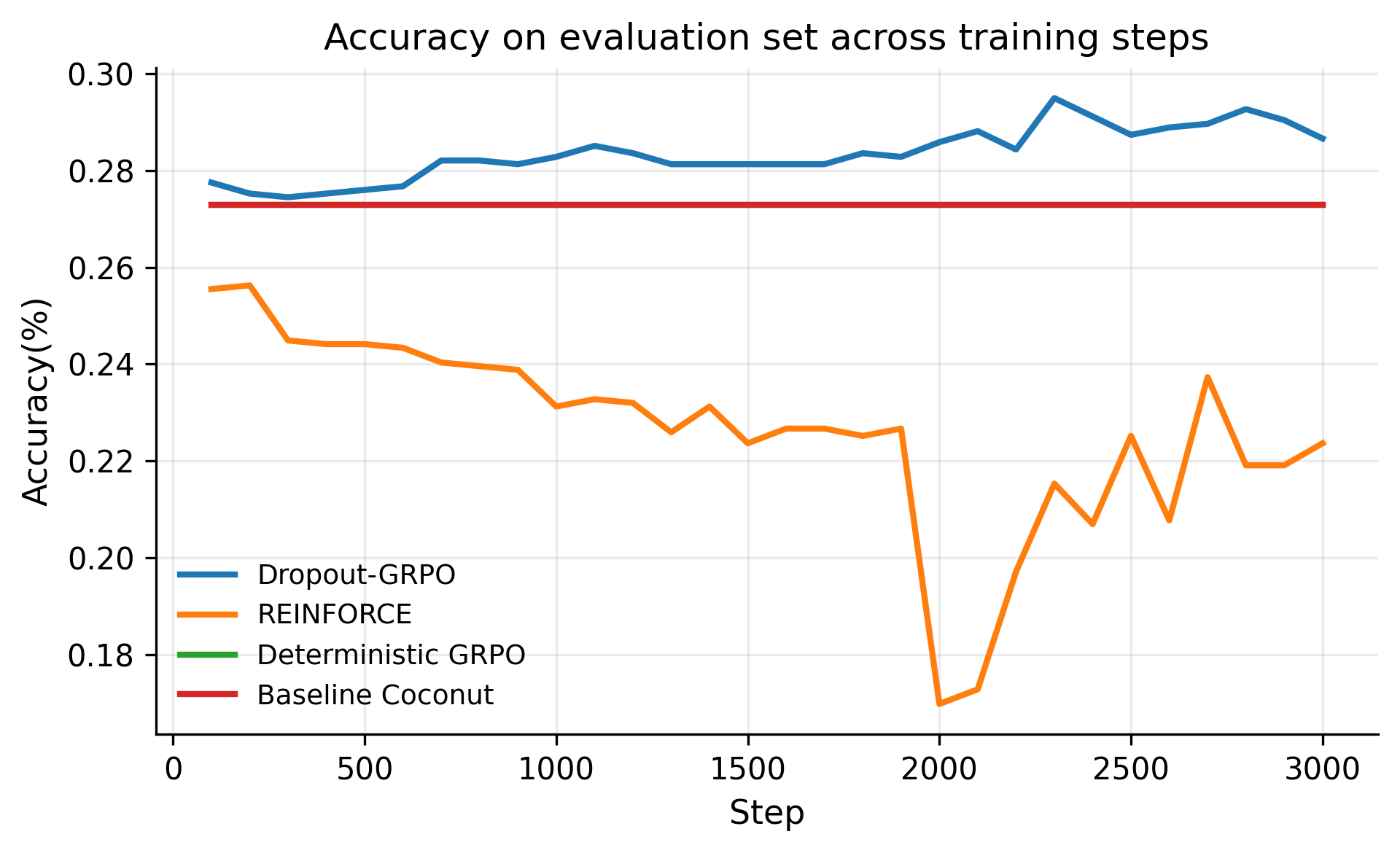}
        \caption{Evaluation accuracy across training steps}
        \label{fig:eval_graph}
    \end{minipage}
    \hfill
    \begin{minipage}[t]{0.48\linewidth}
        \centering
        \includegraphics[width=\linewidth]{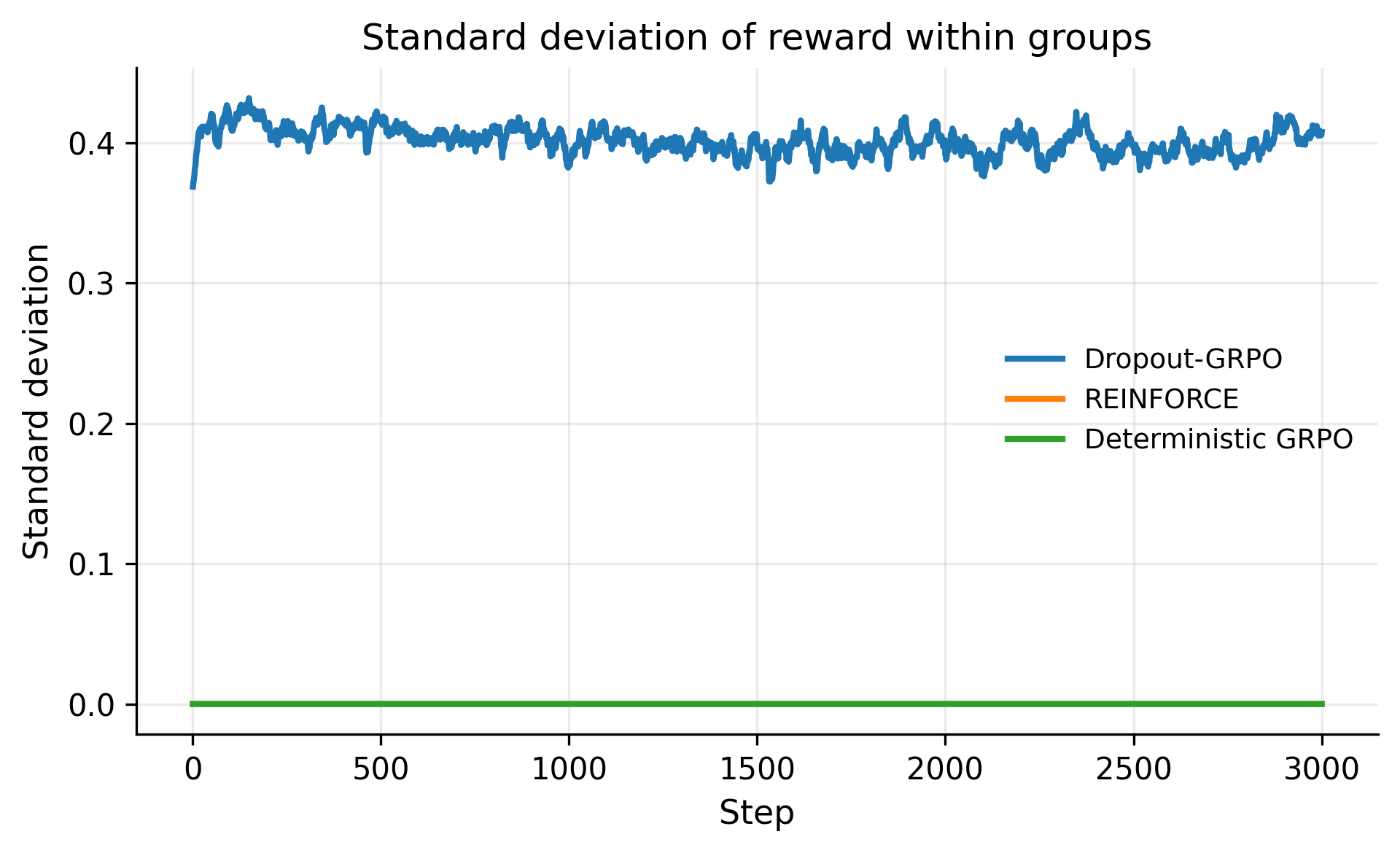}
        \caption{$\sigma_{r}$ within prompt groups over training. Rollout diversity only occurs during Dropout-GRPO.}
        \label{fig:std_r_graph}
    \end{minipage}
\end{figure}

%% file: text/6_limitations.tex
A few caveats apply to our results.

First, \emph{train/deploy mismatch:} inference disables dropout, so the deployed policy is $\pi_\theta(y \mid x)$ rather than the marginal $\bar\pi_\theta$ that GRPO optimizes; we report deterministic-decoding numbers throughout, accepting a bias whose magnitude depends on per-layer Lipschitz constants.

Second, \emph{architectural caveat:} many modern decoder implementations expose only limited dropout surfaces, and large-scale language-model pretraining often disables dropout altogether.\cite{Liu2025DropDropout} Our construction therefore reactivates a mechanism the field has largely retired for regularization. We use dropout here as a source of structured exogenous stochasticity rather than as a regularizer. The proof framework of Section~\ref{sec:theory} still survives this restriction because it requires only that $\xi$ be exogenous and replayable. Broader parameter-side perturbation surfaces (re-introducing residual/MLP dropout at fine-tuning time, or weight noise as in~\citet{Blundell2015WeightUncertainty}) could strengthen the construction and are a natural future-work direction.

Third, \emph{scale and magnitude:} our Qwen2.5-1.5B backbone is small for math benchmarks and our $2.03$-pp improvement is modest in absolute terms, both of which we expect to widen as the backbone is scaled.

Fourth, \emph{dataset availability:} our evaluation scope is bounded by the SFT data distribution available for Coconut-style training: concise per-step reasoning datasets in the \textsc{GSM8K-Aug} mold are scarce, and extending to harder benchmarks such as \textsc{MATH} or \textsc{AIME} requires upstream work on dataset curation rather than modifications to the RL post-training method itself.

A fuller discussion of these points and additional caveats is given in Appendix~\ref{app:extended}.

%% file: text/7_conclusion.tex
We have identified the structural reason that GRPO fails on \textsc{Coconut}-style latent reasoning models and proposed shared-mask variational dropout with mask replay as a minimal fix. The construction admits a Bayesian model-average reading of the marginal policy, and we proved that the resulting surrogate gradient is unbiased on the marginalized objective at $\theta = \theta_\old$ up to a $(1 - 1/K)$ scale, that mask replay provides a common-random-numbers variance reduction over fresh resampling, and that the latent Jacobian is well defined at every depth $T$. Empirically, dropout-GRPO improves a \textsc{Coconut} baseline on \textsc{GSM8K} from $27.29\%$ to $29.01\%$ pass@1, the first successful application of GRPO to continuous latent reasoning using hidden states, to our knowledge. With RL post-training of latent models now feasible, this opens the \textsc{Coconut}-line of work to standard verifier-reward training. Directions for future work are discussed in Appendix~\ref{app:extended}.

%% file: text/A_extended.tex
\subsection{Interpreting the performance gain from Dropout-GRPO}
The empirical improvement of $2.03$ percentage points on \textsc{GSM8K}, while statistically meaningful given our multi-seed variance estimates of $0.24$ percentage points, is modest in absolute terms. The result should be read as a \emph{viability} demonstration rather than as a competitive accuracy claim: prior to this work, group-relative reinforcement learning on \textsc{Coconut}-style latent reasoning produced no gradient signal whatsoever, so accuracy improvements on the order of the SFT baseline's noise floor are the relevant comparison. We expect the gap to widen as the backbone is scaled, latent depth $T$ is extended, and SFT initializations are made cleaner.

\subsection{Latent reasoning vs.\ token chain-of-thought}
\label{A_cot}
Pure-latent reasoning models continue to trail strong token-level chain-of-thought systems on math benchmarks, and our work does not claim otherwise. The contribution is structural: a key tool of the post-training toolkit (group-relative reinforcement learning) is now accessible to latent-reasoning models, which removes a barrier to closing the gap. Whether the gap closes in practice is an empirical question for future work.

\subsection{Memory and reproducibility}
The shared-mask construction is reproducible bit-identically only on the same hardware and PyTorch version; the Philox-based dropout kernel guarantees this within a fixed environment. We store one RNG seed per rollout (a few bytes), keeping the mask-storage cost at $\mathcal{O}(K \cdot \mathrm{seed})$ rather than the naive $\mathcal{O}(K \cdot T \cdot |H|)$ that explicit per-layer mask materialization would require. This makes the method essentially free in memory at the scales we tested.

\subsection{Evaluation scope and dataset availability} 
\textsc{Coconut} curriculum training requires intermediate reasoning chains in a concise per-step format that allows the latent recurrence to absorb each step into a single hidden-state update. \textsc{GSM8K-Aug}~\citep{Deng2024InternalizeCoT} provides this for grade-school arithmetic, but comparable resources for advanced mathematics (\textsc{MATH}, \textsc{MATH-500}, \textsc{AIME}) are not currently available: the prevailing CoT-augmented variants of these datasets emit multi-paragraph derivations, which the Coconut latent representation cannot ingest at the curriculum step granularity. 

We therefore restrict evaluation to benchmarks distributionally aligned with \textsc{GSM8K-Aug} (\textsc{GSM8K}, \textsc{SVAMP}, \textsc{MultiArith}). Extending Dropout-GRPO to harder benchmarks is a complementary direction that depends on either (i) curating concise-step variants of advanced math datasets, or (ii) extending the Coconut SFT methodology to ingest verbose intermediate reasoning, the latter of which aligns with \ref{A_cot}. Neither is a modification to the RL method introduced in this paper.

\subsection{On the use of dropout in modern LLM training}
\label{app:dropout-convention}
The shared-mask dropout construction of Section~\ref{sec:theory-setup} reactivates a mechanism that has fallen out of favor in modern large-language-model training: contemporary foundation models such as Qwen2.5~\citep{qwen2024} are typically trained at scales where dropout is often disabled.\cite{Liu2025DropDropout} The convention shift is well-motivated: dropout's classical role as a generalization regularizer is largely subsumed by data abundance at the scales these models are trained on, where the inductive bias from billions of training tokens supplies the regularization that Bernoulli masking once provided.

We emphasize that our use of dropout is functionally distinct from this regularization role. We do not engage dropout to improve generalization, but as a structured source of \emph{exogenous stochasticity}. The operator's ability to produce bit-identical, replayable parameter perturbations indexed by a discrete RNG seed makes it a natural off-the-shelf instance of the perturbation mechanism required by Theorem~\ref{thm:unbiased} and by the variational interpretation of Section~\ref{sec:theory}. 

\paragraph{Implications for transferability.} The proof framework of Section~\ref{sec:theory} requires only that $\xi$ be exogenous (independent of $\theta$) and replayable (deterministically regenerable from a stored seed). Any architecture supporting any form of replayable parameter-side perturbation is therefore theoretically compatible with the construction, such as weight noise in the style of~\citet{Blundell2015WeightUncertainty}, batch-efficient variants such as Flipout~\citep{Wen2018Flipout}, structured pruning masks, or low-rank perturbations sampled from a fixed distribution. 

Dropout is the cleanest instance available in current decoder configurations, but the construction itself is mechanism-agnostic. Re-introducing residual or MLP dropout at fine-tuning time, or substituting alternative perturbation primitives, are natural directions for extending the present work to architectures and training regimes where attention dropout alone provides insufficient stochasticity surface.

\subsection{Future work}
Several directions appear promising. First, extending dropout-GRPO to chain-of-thought datasets at larger scale, where the supply of high-quality SFT data is abundant. Second, scaling the backbone: $1.5\mathrm{B}$ parameters is small for math benchmarks, and the gain from RL post-training typically widens with model size. Third, multi-GPU validation and throughput: although the implementation includes a DDP-coordinated skip mechanism intended to handle asymmetric filter outcomes across ranks, this code path has not been empirically tested in a multi-GPU configuration. Validating correctness, measuring the all-reduce overhead per step, and tuning for throughput at scale are part of this future-work direction. Fourth, while pure-latent reasoning currently trails token-level chain-of-thought on math, the availability of group-relative reinforcement learning, which is enabled by the construction proposed in this paper, may be a key ingredient in closing that gap.

%% file: text/B_proofs.tex
\subsection{Proof of Lemma~\ref{lem:induction} (well-defined latent gradient)}
\label{app:proof-lemma}

By induction on $t$.

\textit{Base case} $(t = 0)$. By Assumption~\ref{ass:smooth}, $h_0 = \mathrm{embed}_\theta(x)$ is differentiable in $\theta$, giving $\partial h_0 / \partial \theta = \partial \mathrm{embed}_\theta(x) / \partial \theta$ as stated.

\textit{Inductive step.} Assume $\partial h_{t-1} / \partial \theta$ exists almost everywhere. Write $h_t = f(\tilde\theta, h_{t-1})$ as a function of two arguments. By Assumption~\ref{ass:smooth} and the multivariate chain rule,
\[
\frac{\partial h_t}{\partial \theta}
   \;=\;
\frac{\partial f}{\partial \tilde\theta}\bigg|_{h_{t-1}}
   \cdot \frac{\partial \tilde\theta}{\partial \theta}
   \;+\;
\frac{\partial f}{\partial h_{t-1}}\bigg|_{h_{t-1}}
   \cdot \frac{\partial h_{t-1}}{\partial \theta},
\]
with $\partial \tilde\theta / \partial \theta$ the structured Bernoulli rescaling (a deterministic linear map fixed by $\xi$). Both terms exist by Assumption~\ref{ass:smooth} and the inductive hypothesis, yielding the recursion claimed in~\eqref{eq:theory-jacobian}.

A final chain-rule application through the answer-token head, together with differentiability of $\log \pi_\theta(y_\ell \mid x, h_T, y_{<\ell})$ in both $\theta$ and $h_T$ from Assumption~\ref{ass:smooth}, shows that $\nabla_\theta \log \pi_\theta(y \mid x, \xi)$ is well defined almost everywhere.
\qed

\subsection{Proof of Theorem~\ref{thm:unbiased} (unbiased policy-surrogate gradient)}
\label{app:proof-thm}

At $\theta = \theta_\old$, mask replay yields $\rho_{\theta_\old, \ell}^{(k)} \equiv 1$ for every token, so the PPO clip is inactive and the surrogate gradient reduces to the unclipped product. Writing $s^{(k)} := \nabla_\theta \log \pi_\theta(y^{(k)} \mid x, \xi^{(k)})|_{\theta_\old}$,
\[
\nabla_\theta \Loss^{\mathrm{surr}}\big|_{\theta_\old}
   \;=\;
\E\!\left[\frac{1}{K}\sum_{k=1}^K (r^{(k)} - \mu_r)\, s^{(k)}\right].
\]
Expanding the baseline term and using i.i.d.\ rollouts,
\[
\E[\mu_r\, s^{(k)}]
   \;=\;
\frac{1}{K}\sum_{j=1}^K \E[r^{(j)}\, s^{(k)}]
   \;=\;
\frac{1}{K}\E[r^{(k)}\, s^{(k)}],
\]
because for $j \ne k$ the independence factorization combined with the standard zero-mean property of the score function under $\Pi_{\theta_\old}$ gives $\E[r^{(j)}\, s^{(k)}] = \E[r^{(j)}]\,\E[s^{(k)}] = 0$. Substituting,
\[
\E[(r^{(k)} - \mu_r)\, s^{(k)}]
   \;=\;
\Big(1 - \tfrac{1}{K}\Big)\,\E[r^{(k)}\, s^{(k)}].
\]
The right-hand side is the standard REINFORCE policy gradient on the augmented policy $\Pi_\theta$, which by Corollary~\ref{cor:marginal} equals $\nabla J(\theta)|_{\theta_\old}$. Averaging over $k$ preserves this and yields~\eqref{eq:theory-unbiased}.
\qed

\subsection{Proof of Proposition~\ref{prop:crn} (CRN coupling)}
\label{app:proof-prop}

We decompose each estimator into its zero-order value at $\theta = \theta_\old$ plus a Lipschitz perturbation:
\[
\hat g_{\mathrm{CRN}}(\theta;\xi,y)
   \;=\;
\hat g_{\mathrm{CRN}}(\theta_\old;\xi,y)
   \;+\;
\Delta_{\mathrm{CRN}}(\theta;\xi,y),
\qquad
\|\Delta_{\mathrm{CRN}}\| \le L\|\theta - \theta_\old\|,
\]
and analogously for $\hat g_{\mathrm{fresh}}$ with $(\xi', y')$. Because both $(\xi, y)$ and $(\xi', y')$ are drawn from the same base policy $\Pi_{\theta_\old}(\cdot \mid x)$, their zero-order distributions strictly coincide; in particular, $\Var[\hat g_{\mathrm{CRN}}(\theta_\old)] = \Var[\hat g_{\mathrm{fresh}}(\theta_\old)]$ at leading order.

When the same gradient estimator is invoked across multiple evaluation points (e.g.\ across mini-epochs of PPO updates), the fresh-rollout estimator independently resamples $(\xi', y')$ at each call, contributing an additional across-call variance term. The mask-replay (CRN) estimator instead reuses the same $(\xi, y)$ across calls, coupling the score-function estimator across $\theta$ values via shared exogenous noise; the across-call variance contribution vanishes by construction.

This is the standard Common Random Numbers argument~\citep{glasserman1992crn}, specialized here for a score-function estimator. In the trust-region regime $\|\theta - \theta_\old\| \to 0$, the per-sample variance bound is tight at $\theta = \theta_\old$ and degrades only as $O(L \|\theta - \theta_\old\|)$ via the Lipschitz hypothesis.
\qed